\documentclass{article}

\usepackage[preprint]{neurips_2026}

\usepackage[utf8]{inputenc}
\usepackage[T1]{fontenc}
\usepackage{hyperref}
\usepackage{url}
\usepackage{booktabs}
\usepackage{amsfonts}
\usepackage{amsmath}
\usepackage{amssymb}
\usepackage{amsthm}
\usepackage{nicefrac}
\usepackage{microtype}
\usepackage{xcolor}
\usepackage{graphicx}
\usepackage{subcaption}
\usepackage{multirow}
\usepackage{listings}
\usepackage{algorithm}
\usepackage{algorithmic}
\usepackage{wrapfig}

\usepackage{graphicx,amsmath,amsfonts,verbatim}
\usepackage{subcaption}
\usepackage{amssymb}
\usepackage{amsthm}
\usepackage{multirow}
\usepackage{booktabs}
\newtheorem{assumption}{Assumption}
\newtheorem{theorem}{Theorem}

\newcommand{\reals}{{\bf R{}}}  
\newcommand{\prob}{\mathop{\bf P{}}}  
\newcommand{\expect}{\mathop{\bf E{}}}  
\newcommand{\argmin}{\mathop{\rm argmin}}  

\newcommand{\Dkl}{\mathop{D_{\rm kl}}}  

\usepackage{lipsum,authblk}

\graphicspath{{imgs/}}

\title{Probabilistic Recurrent Intention Switching Model}
\author[1,*]{Wenyuan Sheng}
\author[1,2,*]{Hao Zhu}
\author[1,2]{Joschka Boedecker}
\affil[1]{Department of Computer Science, University of Freiburg}
\affil[2]{IMBIT//BrainLinks-BrainTools}
\affil[*]{These authors contributed equally.}

\begin{document}
\maketitle

\begin{abstract}
Inverse reinforcement learning (IRL) recovers reward functions from observed behavior, yet traditional methods assume a single stationary reward that cannot capture goal switching within an episode. Recent multi-intention IRL methods address this by segmenting trajectories, but model intention transitions as either a memoryless Markov chain or via manual state augmentation with a fixed history window. We propose the Probabilistic Recurrent Intention Switching Model (PRISM), which replaces both mechanisms with a lightweight recurrent network that maps observation history to a per-step intention distribution. We prove that the resulting EM objective decomposes exactly into independent per-intention reward subproblems, each solvable in closed form, yielding an $\mathcal{O}(nK)$ E-step with no variational approximation. We evaluate PRISM on a non-Markovian gridworld, a mouse labyrinth, and BridgeData~V2 robotic manipulation, the first large-scale robotic application of multi-intention IRL. Across all settings PRISM achieves the highest held-out log-likelihood while recovering nameable, temporally coherent intentions from unlabeled demonstrations, suggesting that discrete goal switching is present in both biological and artificial agents.
\end{abstract}

\section{Introduction}

Demonstrations from any goal-directed agent, whether biological or artificial, contain latent goal switches. A mouse navigating a labyrinth alternates between seeking water and returning home; a teleoperated robot arm switches between reaching, grasping, and placing within a single demonstration; even human drivers shift between lane-keeping, overtaking, and parking. Recovering the reward functions that drive each goal, the moments at which the agent switches between them, and the aspects of the history that trigger these switches is essential for understanding complex sequential behavior. Yet standard inverse reinforcement learning (IRL) assumes a single stationary reward, conflating multiple objectives into one function that cannot distinguish between them.

Three lines of work have attempted to lift this stationarity assumption. Dynamic IRL (DIRL)~\citep{ashwood2022dynamic} models the reward as a smoothly time-varying combination of spatial goal maps, but its continuity assumption conflicts with evidence that animals switch between \emph{discrete} strategies~\citep{ashwood2022mice}, and its inference requires $T$ separate Bellman solves, one per timestep. Hierarchical Inverse Q-Learning (HIQL)~\citep{zhu2024multiintention} replaces smooth variation with a first-order Markov chain over intentions, but the memoryless transition model cannot capture switching driven by cumulative experience such as fatigue or satiation, and its E-step requires the Baum--Welch forward--backward algorithm at $\mathcal{O}(nK^2)$ cost per trajectory. SWIRL~\citep{ke2025inverse} adds state-dependent transition kernels and history-augmented rewards, but represents history via state augmentation: with history length $L$ on $|\mathcal{S}|$ states the augmented state space grows as $|\mathcal{S}|^L$ (e.g., $L{=}2$ on 127 states yields $127^2 {=} 16{,}129$ augmented states), limiting $L$ to small values in practice. None of these approaches scale naturally beyond small tabular environments.

We propose the \emph{Probabilistic Recurrent Intention Switching Model} (PRISM), a framework that absorbs all temporal complexity into a single lightweight recurrent neural network. At each timestep the network reads an observation from the environment, updates a hidden state that summarizes the trajectory so far, and outputs a soft assignment over a finite set of intentions. This assignment is combined with per-intention Boltzmann policy likelihoods to form a per-step posterior responsibility over intentions. We prove that the resulting expectation-maximization (EM) objective decomposes exactly into independent per-intention reward subproblems, each solvable in closed form via inverse action-value iteration (IAVI)~\citep{NEURIPS2020_a4c42bfd}, requiring no variational approximation. The posterior factorizes into independent per-step terms, yielding an $\mathcal{O}(nK)$ E-step. With approximately 50K trainable parameters in a single-layer RNN, PRISM trains in minutes on a laptop GPU and produces human-interpretable reward maps without manual specification of the temporal horizon.

We evaluate PRISM on three domains of increasing complexity, each testing a different property of the framework. Our central application is the \emph{127-node mouse labyrinth}~\citep{10.7554/eLife.66175}, where PRISM recovers three intentions (water-seeking, homing, exploration) aligned with known biological drives, matching the modes identified by both DIRL~\citep{ashwood2022dynamic} and SWIRL~\citep{ke2025inverse} on the same dataset while achieving higher held-out log-likelihood. A \emph{frustration gridworld} with a hidden counter provides controlled validation that PRISM captures provably non-Markovian switching, which is essential for modeling history-dependent behavioral states such as satiation and fatigue. The \emph{BridgeData~V2 robotic manipulation dataset}~\citep{walke2023bridgedata} tests whether the method generalizes beyond neuroscience: without any supervision, PRISM discovers four temporally coherent manipulation phases (approach, grasp, carry, idle) from human-teleoperated demonstrations. Together, these experiments suggest that discrete intention switching is present in both biological and artificial agents, and that the reward maps PRISM recovers provide a useful lens for interpreting the latent goals behind complex sequential behavior.

In summary, our contributions are:
(i)~the PRISM framework with a proven EM decomposition and closed-form reward recovery;
(ii)~an $\mathcal{O}(nK)$ E-step via posterior factorization, compared to the $\mathcal{O}(nK^2)$ forward--backward pass required by Markov-chain-based alternatives;
(iii)~to our knowledge, the first application of multi-intention IRL to a large-scale robotic manipulation dataset;
(iv)~recovery of nameable, interpretable intentions across three domains without supervision.

\section{Related Work}

\paragraph{Inverse reinforcement learning.}
IRL recovers a reward function from demonstrations, assuming the expert maximizes long-term return. The problem was formalized by \citet{ng2000algorithms}, who showed it is inherently ill-posed. Maximum entropy IRL~\citep{ziebart2008maximum} and its causal variant~\citep{ziebart2010modeling} resolve this ambiguity via entropy regularization. \citet{NEURIPS2020_a4c42bfd} derived a closed-form reward solution via inverse action-value iteration (IAVI), which we use as the inner-loop solver in our EM framework. For a comprehensive survey, see \citet{arora2021survey}.

\paragraph{Multi-intention and dynamic IRL.}
Standard IRL assumes a single fixed reward. In practice, agents switch between distinct strategies within an episode~\citep{ashwood2022mice}. \citet{babes2011apprenticeship} cluster entire trajectories by intention but do not allow within-episode switching. Bayesian nonparametric methods~\citep{dimitrakakis2012bayesian, surana2014bayesian} avoid fixing the number of intentions but scale poorly. DIRL~\citep{ashwood2022dynamic} models reward as a continuously time-varying combination of goal maps; HIQL~\citep{zhu2024multiintention} uses a first-order Markov chain; SWIRL~\citep{ke2025inverse} adds state-dependent transitions with fixed-window history augmentation. PRISM combines the strengths of all three: discrete switching like HIQL but with memory; history-dependent like SWIRL but learned end-to-end; data-driven like DIRL but with discrete intentions and closed-form reward recovery. 

\paragraph{Hierarchical imitation learning and option discovery.}
The options framework~\citep{sutton1999between} decomposes policies into temporally extended sub-policies. CompILE~\citep{kipf2019compile} segments demonstrations into composable latent skills, and play-data approaches~\citep{lynch2020latent} discover reusable motor primitives from unstructured teleoperation. These methods recover \emph{policies} or \emph{skills}; PRISM instead recovers per-intention \emph{reward functions}, which are more compact, transferable across dynamics, and directly interpretable. PRISM can be seen as performing ``inverse option discovery,'' recovering the reward structure that generates the behavioral options.

\paragraph{Imitation learning at scale.}
Behavioral cloning~\citep{ross2011reduction}, GAIL~\citep{ho2016gail}, and recent generalist robot policies such as RT-1~\citep{brohan2023rt1} learn policies directly from demonstrations without recovering reward structure. PRISM complements these approaches: the intention segmentation it produces could serve as a pre-processing step for skill-conditioned imitation learning.

\paragraph{Interpretable reward and decision models.}
Reward decomposition~\citep{juozapaitis2019explainable} and programmatic RL~\citep{verma2018programmatically} seek to make RL decisions transparent. PRISM contributes to this space by recovering discrete, nameable intentions and per-intention reward maps from unlabeled demonstrations.

\paragraph{Positioning.}
PRISM is a table-based, model-based (requires known or estimated transition model $P$), offline multi-intention IRL method. It operates on fully observed demonstration trajectories: all states, actions, and observations are available before inference begins. Its EM algorithm alternates between closed-form reward recovery via IAVI and gradient-based updates to the recurrent intention network, without end-to-end differentiation through the Bellman equation.

\section{Background}

\paragraph{Notation.}
We consider an MDP $\langle \mathcal{S}, \mathcal{A}, P, \gamma \rangle$ with finite state space $\mathcal{S}$, finite action space $\mathcal{A}$, transition function $P(s' \mid s, a)$, and discount factor $\gamma \in [0,1)$. A policy $\pi(a \mid s)$ specifies a distribution over actions in each state.

\paragraph{Hidden-intention MDP.}
\label{sec:hi_mdp}
A Hidden-Intention MDP (HI-MDP) extends the standard MDP with a latent intention space $\mathcal{Z}$ and is denoted $\langle \mathcal{Z}, \mathcal{S}, \mathcal{A}, P, \{r_z\}_{z \in \mathcal{Z}}, \gamma \rangle$. At each timestep $t$, the expert operates under a latent intention $z_t \in \mathcal{Z}$ and receives reward according to $r_{z_t} \colon \mathcal{S} \times \mathcal{A} \to \reals$. The expert also perceives an observation $\varphi_t \in \reals^m$ from the environment; the observation need not be in one-to-one correspondence with the state. For each trajectory $\xi = \{(s_1, a_1), \dots, (s_n, a_n)\}$ there is a corresponding observation sequence $\psi = \{\varphi_1, \dots, \varphi_n\}$, and we denote the set of all such sequences by $\mathcal{O}$. The mechanism by which observations influence intention assignments is left unspecified in the HI-MDP; PRISM instantiates this mechanism via a learned gating function $f_\theta$ (\S\ref{sec:methods-prism}).

\paragraph{Inverse action-value iteration.}
Given demonstrations $\mathcal{D}$ in an MDP with known $P$, IAVI~\citep{NEURIPS2020_a4c42bfd} formulates IRL as maximum likelihood estimation under a Boltzmann policy and solves for the reward in closed form via least squares. We use IAVI as the inner-loop solver in our EM framework; the full optimization program is given in \S\ref{app:iavi}.

\section{Methods}
\label{sec:methods}
\subsection{Probabilistic Recurrent Intention Switching Model}
\label{sec:methods-prism}
\begin{wrapfigure}{r}{0.48\textwidth}
    \vspace{-14pt}
    \centering
    \includegraphics[width=0.48\textwidth]{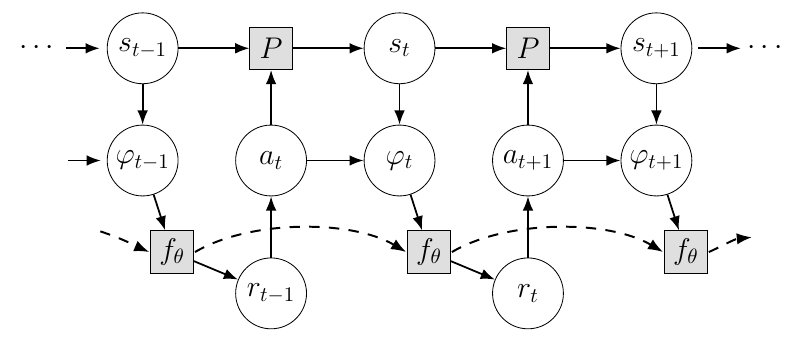}
    \caption{Probabilistic graphical model of the expert's decision process. Dashed lines represent the recurrent connection of $f_\theta$.}
    \label{fig:pgm}
    \vspace{-8pt}
\end{wrapfigure}

We formulate the multi-intention IRL problem under three assumptions.

\begin{assumption}\label{asm:boltzmann}
Each expert demonstration step is generated according to a Boltzmann-optimal policy under one of the reward functions in a $K$-dimensional finite set $\mathcal{R} = \{r_1, \dots, r_K\}$.
\end{assumption}

\begin{assumption}\label{asm:observation}
Each trajectory $\xi$ in the demonstration set $\mathcal{D}$ is accompanied by an observation sequence
$\psi = \{\varphi_1, \dots, \varphi_n\} \in \mathcal{O}$,
as defined in \S\ref{sec:hi_mdp}. The observation $\varphi_i$ may encode any information available at step $i$, including the state, the action, or features derived from either.
\end{assumption}

\begin{assumption}\label{asm:control}
The soft intention assignment at each demonstration step is produced by a parametric function
$f_\theta \colon \reals^m \to \Delta^K$, where
${f_\theta(\varphi_i)}_k$ denotes the weight assigned to
intention $k$ given observation $\varphi_i$, for
$k = 1, \dots, K$. The function $f_\theta$ may maintain
internal state across steps within a trajectory, as realized by recurrent neural networks.
\end{assumption}

Under these assumptions the expert's decision process is fully specified by $\Theta = \{\theta, \mathcal{R}\}$, and the resulting decision network can be represented with a probabilistic graphical model (Figure~\ref{fig:pgm}). We refer to the resulting framework as the \emph{Probabilistic Recurrent Intention Switching Model} (PRISM). The PRISM inference problem consists of determining (1)~a set of reward functions and (2)~the intention index for each demonstration step that best jointly explain the observed expert behavior. An expectation-maximization (EM) algorithm can be devised to iteratively learn $\Theta$. For convenience, we introduce $\eta = \{z_1, \dots, z_n\}$ as the predicted sequence of intention indices for trajectory $\xi$ and corresponding observations $\psi$. Then each iteration of the EM process maximizes the auxiliary function:
\begin{equation}
\label{eq:em_mle}
\text{maximize} \quad J(\Theta^+ \mid \Theta) = \expect_{(\xi,\psi) \sim (\mathcal{D},\mathcal{O}),\, \eta}\log \prob(\xi, \eta \mid \psi, \Theta^+),
\end{equation}
where $\Theta^+$ is the optimization variable and $(\mathcal{D}, \mathcal{O}), \Theta$ are the problem data.

\begin{theorem}\label{thm:decomp}
Solving problem \eqref{eq:em_mle} is equivalent to solving a sequence of independent optimization problems: first maximize over the intention network parameters $\theta^+$,
\begin{equation}\label{eq:theta_opt}
\begin{array}{ll}
    \mbox{\rm maximize (over $\theta^+$)} & \expect_{(\xi,\psi) \sim (\mathcal{D},\mathcal{O})}\left(\sum_{k=1}^K \sum_{i=1}^n \prob(z_i{=}k \mid \xi, \psi, \Theta) \log {f_{\theta^+}(\varphi_i)}_k\right),
\end{array}
\end{equation}
and then maximize independently over each reward $r_k^+ \in \mathcal{R}^+$,
\begin{equation}\label{eq:reward_opt}
\begin{array}{ll}
\mbox{\rm maximize (over $r_k^+$)} & \expect_{(\xi,\psi) \sim (\mathcal{D},\mathcal{O})}\left(\sum_{i=1}^n \prob(z_i{=}k \mid \xi, \psi, \Theta) \log \pi_{r_k^+}(a_i \mid s_i)\right) \\[6pt]
\mbox{\rm subject to} & \pi_{r_k^+}(s, a) = \exp\bigl(Q(s,a) - \log \sum_{a' \in \mathcal{A}} \exp Q(s, a')\bigr) \\
& Q(s,a) = r_k^+(s,a) + \gamma \sum_{s'} P(s' \mid s,a)\max_{a' \in \mathcal{A}} Q(s',a')\\
& s \in \mathcal{S},\quad a \in \mathcal{A}.
\end{array}
\end{equation}
\end{theorem}
This decomposition is exact, requiring no variational approximation. Each reward subproblem \eqref{eq:reward_opt} reduces to a weighted IAVI problem solvable in closed form. Detailed proof is given in \S~\ref{app:proof}.

\paragraph{Posterior factorization and complexity.}
A key structural consequence is that, conditioned on $(\xi, \psi, \Theta)$, the intention variables $\{z_1, \dots, z_n\}$ are mutually independent. The posterior decomposes as $\prob(\eta \mid \xi, \psi, \Theta) = \prod_{i=1}^n \prob(z_i \mid \xi, \psi, \Theta)$, where each per-step responsibility is the normalized product of the gating output and the per-intention policy:
\begin{align}
    \prob(z_i {=} k \mid \xi, \psi, \Theta)
    &= \frac{{f_\theta(\varphi_i)}_k \, \pi_{r_k}(a_i \mid s_i)}
           {\sum_{j=1}^K {f_\theta(\varphi_i)}_j \, \pi_{r_j}(a_i \mid s_i)},
    \label{eq:posterior}
\end{align}
where $f_\theta$ assigns soft responsibility over intentions given the observed context $\varphi_i$, and $\pi_{r_k}$ scores how well the observed action $a_i$ is explained under intention $k$. Since the intention variables decouple across time steps, the E-step requires only $\mathcal{O}(nK)$ evaluations per trajectory.

\subsection{Training Objective}
\label{sec:criterion}
The training objective for $f_\theta$ combines the negative log-likelihood from the M-step with temporal smoothness penalties. Writing $w_{i,k}$ for the per-step responsibility in Eq.~\eqref{eq:posterior}:
\begin{equation}\label{eq:loss}
\mathcal{L} \;=\; \mathcal{L}_{\mathrm{NLL}} \;+\; \lambda_{\ell_1}\,\mathcal{L}_{\ell_1} \;+\; \lambda_{\mathrm{kl}}\,\mathcal{L}_{\mathrm{kl}},
\end{equation}
where
\begin{align*}
\mathcal{L}_{\mathrm{NLL}} &= -\expect_{(\xi,\psi) \sim (\mathcal{D},\mathcal{O})} \sum_{i=1}^n \sum_{k=1}^K w_{i,k} \log {f_\theta(\varphi_i)}_k, \\
\mathcal{L}_{\ell_1} &= \expect_{(\xi,\psi) \sim (\mathcal{D},\mathcal{O})} \sum_{i=2}^n \sum_{k=1}^K w_{i,k} \bigl|{f_\theta(\varphi_i)}_k - {f_\theta(\varphi_{i-1})}_k\bigr|, \\
\mathcal{L}_{\mathrm{kl}} &= \expect_{(\xi,\psi) \sim (\mathcal{D},\mathcal{O})} \sum_{i=2}^n \Dkl\bigl(f_\theta(\varphi_{i-1}) \,\|\, f_\theta(\varphi_i)\bigr).
\end{align*}
where the smoothness terms are computed over consecutive time steps within each trajectory. The $\ell_1$-penalty suppresses rapid intention switches; the KL-divergence term penalizes distributional shift more gently, preserving distributional shape. The two are complementary: $\ell_1$ provides sparse switching while KL preserves smooth transitions. Both weights $\lambda_{\ell_1}, \lambda_{\mathrm{kl}} \geq 0$ are set in \S\ref{exp:labyrinth}. We note that these penalties act as discriminative regularizers rather than a probabilistic prior over intention sequences. The RNN hidden state provides implicit temporal modeling, while the regularizers encourage the output distribution to vary smoothly.

\begin{algorithm}[t]
\caption{PRISM: Probabilistic Recurrent Intention Switching Model}
\label{alg:prism}
\begin{algorithmic}[1]
\REQUIRE Expert demonstrations and observations $(\mathcal{D}, \mathcal{O})$, intention network $f_\theta$, reward set dimension $K$.
\STATE \textbf{initialize} $\theta$, $r_1, \dots, r_K$.
\REPEAT
\STATE Compute posterior $\prob(z_i {=} k \mid \xi, \psi, \Theta)$ for each demonstration step via Eq.~\eqref{eq:posterior}.
\STATE Update $\theta$ by applying SGD on $(\mathcal{D}, \mathcal{O})$ with loss $\mathcal{L}$ (Eq.~\eqref{eq:loss}): $\theta := \argmin_\theta \mathcal{L}(\theta)$.
\FOR{$k = 1, \dots, K$}
\STATE Update $r_k$ by solving problem~\eqref{eq:reward_opt} via IAVI with each demonstration weighted by\\ $\prob(z_i {=} k \mid \xi, \psi, \Theta)$.
\ENDFOR
\UNTIL{stopping criterion is satisfied.}
\end{algorithmic}
\end{algorithm}

\subsection{Intention Network Architecture}
The intention network maps a sequence of observations $(\varphi_1, \dots, \varphi_n)$ to a per-step distribution over $K$ intentions. Each observation $\varphi_t$ is embedded into a $d$-dimensional vector; the sequence of embeddings is processed by a recurrent encoder (RNN, LSTM, or Transformer); and each encoder output is projected to $K$ logits followed by a softmax to produce $f_\theta(\varphi_t) \in \Delta^K$. In the tabular experiments we set $\varphi_t = (s_t, a_t)$ and use learned state and action embedding matrices; in visual domains $\varphi_t$ may be the output of a pretrained encoder. A single-layer IntentionRNN with $d{=}128$ and $K{=}4$ has approximately 50K trainable parameters.

\section{Experiments}

\subsection{Frustration Gridworld}\label{exp:gridworld}

This experiment validates PRISM on a controlled environment where the ground-truth intention-switching mechanism is known and provably non-Markovian.

\paragraph{Setup.}
A $5{\times}5$ gridworld with five actions (up, down, left, right, stay) and stochastic dynamics (90\% success rate). The expert has two policies: $\pi_{\text{goal}}$ toward $(4,4)$ and $\pi_{\text{abandon}}$ toward the origin. A hidden frustration counter $c$, initially zero, increments at each barrier encounter; the switching probability is $\min\{0.15c, 0.9\}$, making intention transitions strictly non-Markovian. The counter resets upon switching. We generate 1024 trajectories and evaluate with 5-fold cross-validation.

\paragraph{Results.}
We compare PRISM ($K{=}2$) against HIQL~\citep{zhu2024multiintention}, IAVI~\citep{NEURIPS2020_a4c42bfd}, maximum causal entropy IRL~\citep{ziebart2010modeling}, and maximum entropy IRL~\citep{ziebart2008maximum}. PRISM achieves the highest test log-likelihood with the smallest variance across folds (Figure~\ref{fig:gridworld}a) and the lowest expected value difference under both intentions (Table~\ref{tab:evd}). HIQL follows closely in log-likelihood but shows higher EVD under the abandon intention, attributable to its Markov model being unable to capture the cumulative counter. Figure~\ref{fig:gridworld}b shows that PRISM's recovered state-value heatmaps closely match the ground truth under both intentions; full comparisons with all baselines are in \S\ref{app:gridworld}. Figure~\ref{fig:gridworld}c displays the temporal intention posterior for a representative trajectory overlaid with the hidden frustration counter: PRISM's posterior tracks the accumulating frustration, switching sharply after repeated barrier encounters.

\begin{figure}
    \centering
    \raisebox{-0.5\height}{\includegraphics[width=0.31\linewidth]{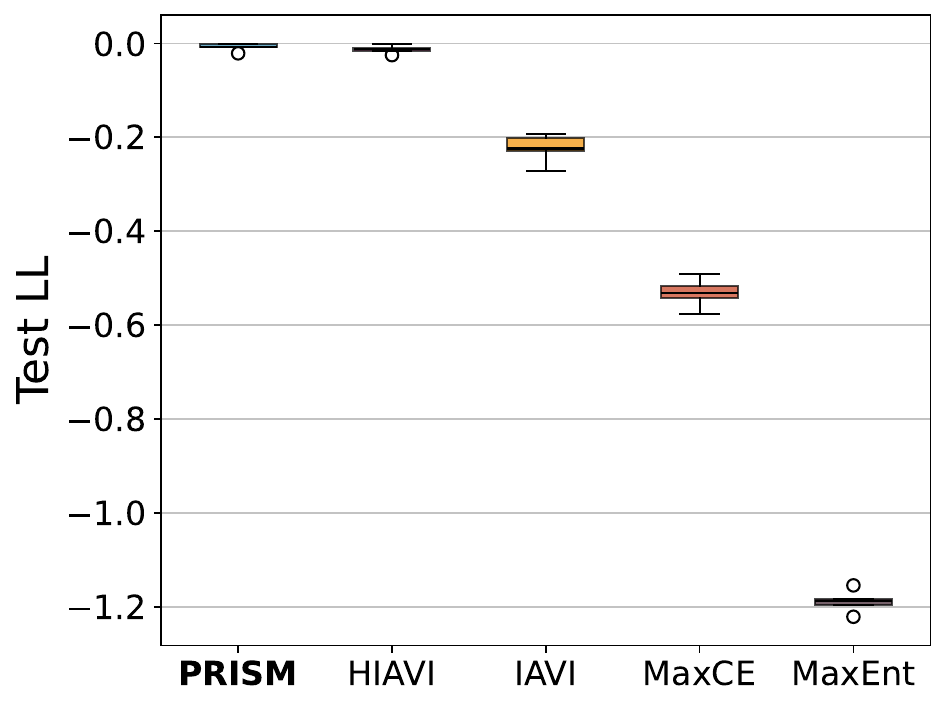}}\hfill
    \raisebox{-0.5\height}{\includegraphics[width=0.32\linewidth]{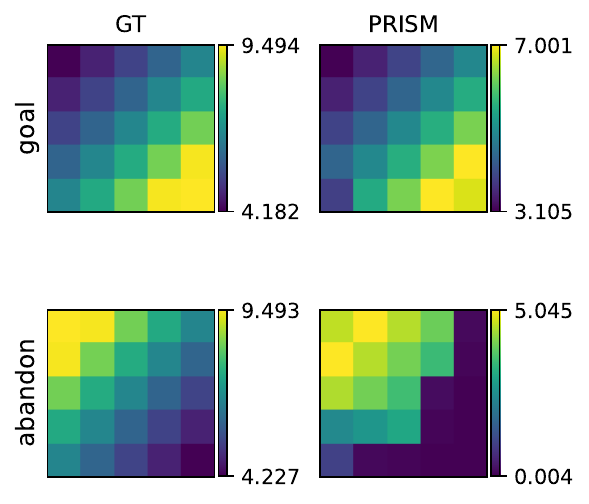}}\hfill
    \raisebox{-0.52\height}{\includegraphics[width=0.34\linewidth]{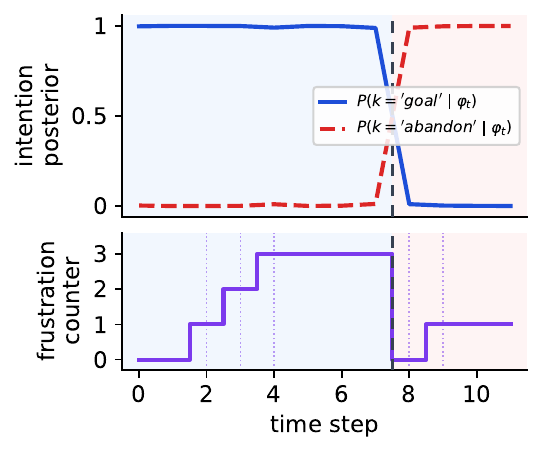}}
    \caption{Frustration gridworld. \textbf{(a)}~Test log-likelihood: PRISM achieves the highest score with the smallest variance. \textbf{(b)}~State-value heatmaps (GT vs PRISM) under the goal and abandon intentions. \textbf{(c)}~Temporal intention posterior (top) overlaid with the hidden frustration counter (bottom) for a representative trajectory; PRISM's posterior tracks the accumulating frustration and switches sharply after repeated barrier encounters. Full comparisons with all baselines are in \S\ref{app:gridworld}.}
    \label{fig:gridworld}
\end{figure}

\subsection{Mouse Labyrinth Navigation}\label{exp:labyrinth}

Moving from simulation to real biological data, we test whether PRISM recovers interpretable intentions aligned with known biological drives.

\paragraph{Dataset and benchmarks.}
We use the freely moving mouse navigation dataset from \citet{10.7554/eLife.66175}, the same benchmark used by DIRL~\citep{ashwood2022dynamic}, HIQL~\citep{zhu2024multiintention}, and SWIRL~\citep{ke2025inverse}. Ten water-restricted mice explored a 127-node labyrinth for 7 hours in darkness; water was available at a designated end node but collectible at most once per 90~seconds, creating a non-Markovian reward structure. Following \citet{ke2025inverse}, we segment raw sequences into 238 trajectories of 500~steps each. Models are evaluated via 5-fold cross-validation. We compare PRISM ($K{=}3$, IntentionRNN) against IAVI, SWIRL~(S-2), and maximum causal entropy IRL. Smoothness penalties are set to $\lambda_{\ell_1}{=}2.22$, $\lambda_{\mathrm{kl}}{=}1.48$; in all other experiments both are zero.

\paragraph{Behavior prediction.}
PRISM achieves the highest test log-likelihood ($-0.65$), outperforming SWIRL S-2 ($-0.73$) on the same data (Figure~\ref{fig:labyrinth}a). The improvement is notable because both methods model history-dependent intention dynamics: PRISM obtains a better fit while learning the relevant history horizon end-to-end, without manual state augmentation. Figure~\ref{fig:labyrinth}b shows that all three architectures (RNN, LSTM, Transformer) improve monotonically with $K$; we adopt the vanilla RNN as default for its simpler dynamics and stable performance. We select $K{=}3$ because the labyrinth task contains only two explicit behavioral events (water-seeking and homing), and the plateau in test log-likelihood beginning at $K{=}4$ (Figure~\ref{fig:labyrinth}b) indicates that three intentions already capture the dominant behavioral structure, with the third corresponding to exploration. This setting also allows direct comparison with SWIRL~\citep{ke2025inverse}, which reports its primary results at $K{=}3$.

\paragraph{Recovered reward maps and segmentation.}
Figure~\ref{fig:labyrinth}c shows the greedy policy and per-state action confidence for each intention, where color depth indicates how decisively the greedy action dominates over alternatives. Under \emph{water}, the greedy policy traces a clear path from the entrance toward the water port, with high confidence along the main corridor. Under \emph{home}, the major branching nodes consistently point toward the entrance, forming convergent flow inward. Under \emph{explore}, actions are distributed across peripheral nodes with no dominant target and lower overall confidence, reflecting the diffuse nature of exploratory behavior. These three modes are consistent with those identified by SWIRL~\citep{ke2025inverse}.
Figure~\ref{fig:labyrinth}d shows that hybrid regularization produces temporally coherent segments whose boundaries align with behavioral events (water-port visits, home visits). The four configurations shown in Figure~\ref{fig:labyrinth}d span the regularization spectrum from unconstrained to fully penalized.

\begin{figure}
    \centering
    \begin{minipage}[c]{0.27\textwidth}
        \centering
        \includegraphics[width=0.95\linewidth]{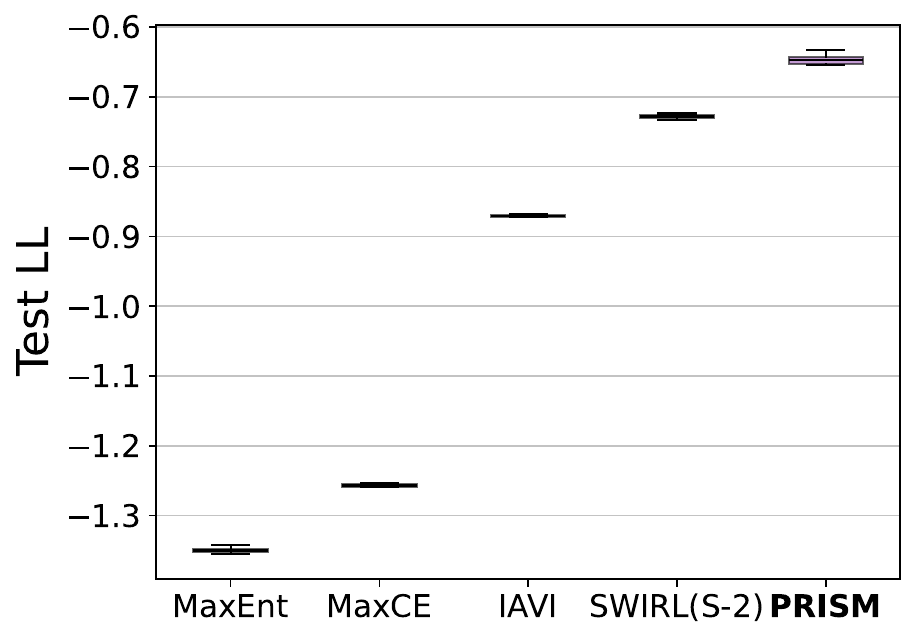}\\[2pt]
        \includegraphics[width=\linewidth]{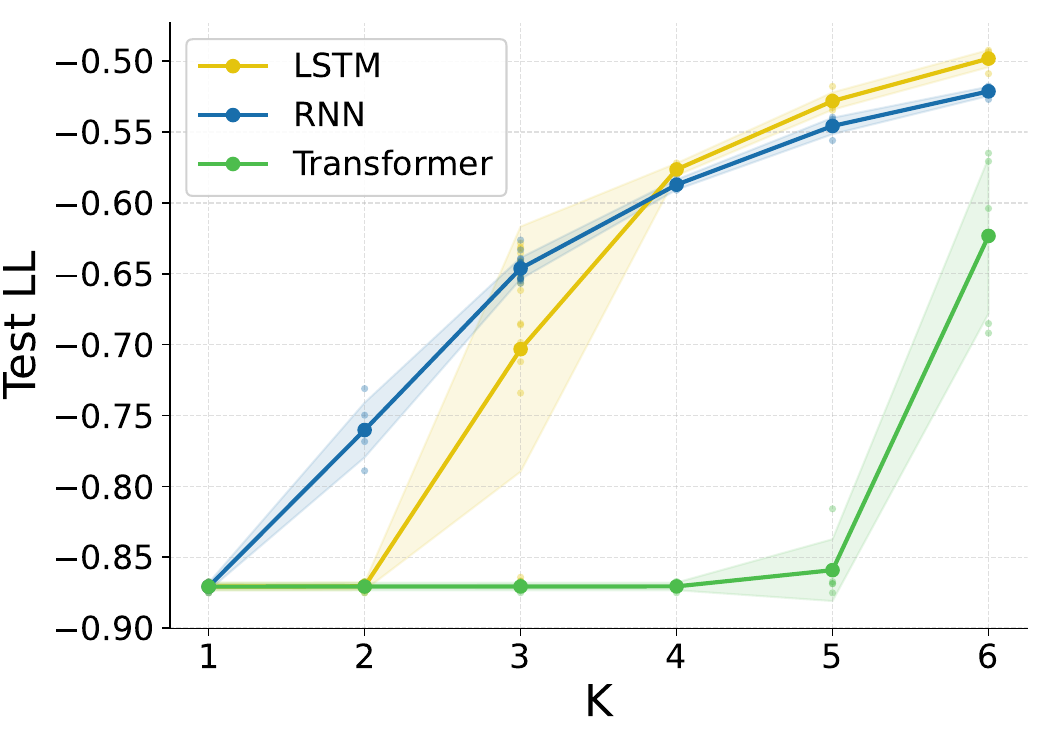}
    \end{minipage}
    \hfill
    \begin{minipage}[c]{0.70\textwidth}
        \centering
        \includegraphics[width=\linewidth]{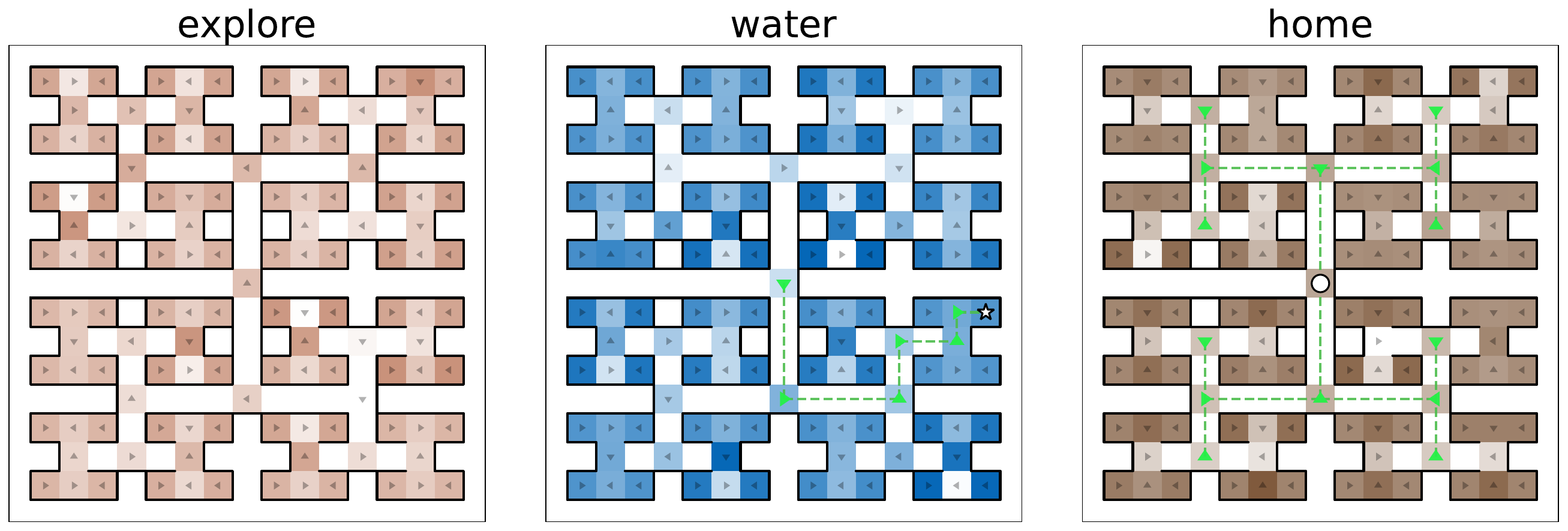}
    \end{minipage}

    \includegraphics[width=0.97\linewidth]{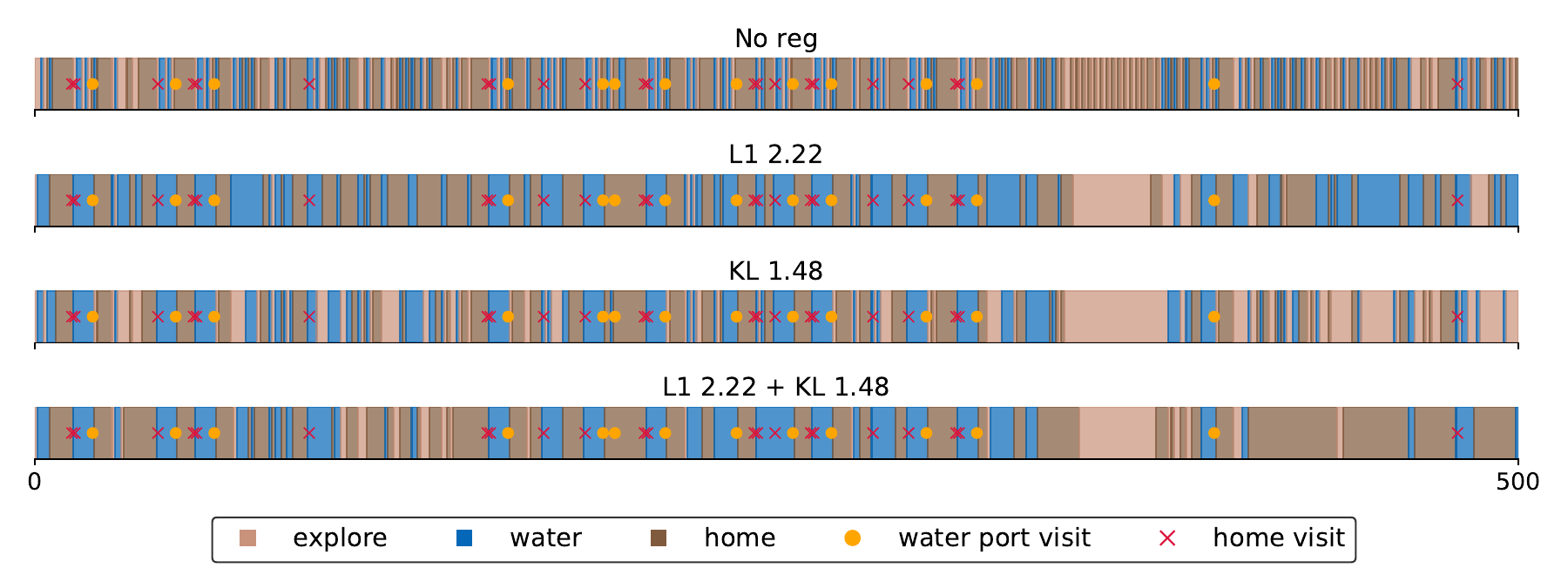}
    \caption{Labyrinth results (PRISM, $K{=}3$, IntentionRNN, hybrid regularization, 238 mouse trajectories). \textbf{(a)}~Test log-likelihood: PRISM outperforms all baselines. \textbf{(b)}~Test log-likelihood vs $K$ for three architectures. \textbf{(c)}~Recovered reward maps and greedy policy flow for the three inferred intentions (star: water port; circle: entrance). \textbf{(d)}~Temporal intention segmentation under four regularization configurations. Orange dots: water-port visits; red crosses: home visits.}
    \label{fig:labyrinth}
\end{figure}

Table~\ref{tab:timing} reports wall-clock timing. PRISM converges in ${\sim}140$ EM iterations on the labyrinth (${\sim}5$~min total) on a NVIDIA RTX 3060 mobile GPU, and ${\sim}80$ iterations on Bridge~V2 (${\sim}12$~min), on a NVIDIA RTX 5060Ti GPU.

\begin{table}
\centering
\small
\begin{tabular}{lccccc}
\toprule
\textbf{Dataset}
    & \textbf{E-step}
    & \multicolumn{2}{c}{\textbf{M-step (per latent)}}
    & \textbf{Total / iter}
    & \textbf{Conv.\ iters}\\
\cmidrule(lr){3-4}
    & (posterior) & IAVI & Intention net & & \\
\midrule
Labyrinth   & 0.03\,s & 0.60\,s & 0.09\,s & 1.92\,s & $\sim$140 \\
Bridge V2   & 4.10\,s & 1.53\,s & 2.90\,s & 8.54\,s & $\sim$80\\
\bottomrule
\end{tabular}
\caption{Wall-clock time per EM iteration.}
\label{tab:timing}
\end{table}

\subsection{BridgeData~V2: Robotic Manipulation}\label{sec:bridge-v2}

We push PRISM to a qualitatively new regime with high-dimensional visual observations and no prior knowledge of intentions, which constitutes, to our knowledge, the first application of multi-intention IRL to a large-scale robotic manipulation dataset.

\paragraph{Motivation.}
In the labyrinth, recovered intentions align with known biological drives. BridgeData~V2 contains no such prior: each trajectory is a human-teleoperated manipulation sequence. If PRISM nonetheless discovers temporally coherent segments corresponding to recognizable manipulation phases, it provides evidence that intention structure is a general property of goal-directed behavior rather than an artifact specific to biological agents, and that the recovered reward maps can help explain the purpose behind human-operated demonstrations.

\paragraph{Continuous-to-discrete pipeline.}
Each $256{\times}256$ RGB frame is encoded by a frozen pretrained visual encoder; 7D continuous actions are retained raw. $k$-means clustering produces $|\mathcal{S}|$ state tokens and $|\mathcal{A}|{=}32$ action tokens. We compare DINOv2~\citep{oquab2024dinov2} (self-supervised; ViT-S/B/L) and SigLIP2~\citep{tschannen2025siglip} (vision-language; ViT-B). Table~\ref{tab:state_granularity} reports discretization statistics across encoders and granularities. Discretization quality is governed by encoder architecture rather than scale: DINOv2-S, -B, and -L produce nearly identical statistics, so we adopt the smallest variant to minimize compute. The critical gap is between DINOv2 and SigLIP2: at $|\mathcal{S}|{=}2048$, SigLIP2-B collapses visually distinct frames into shared tokens (11.8 average revisits vs 7.1 for DINOv2-S), reducing the action-level state discrimination that IRL requires. At our default granularity ($|\mathcal{S}|{=}2048$, DINOv2-S), approximately 75\% of visited states receive multiple observations, providing adequate support for transition estimation. This encoder advantage is reflected in test log-likelihood: DINOv2 consistently outperforms SigLIP2 across all configurations (Figure~\ref{fig:bridge_results}).

\begin{table}
\centering
\small
\caption{Effect of $k$-means state granularity on trajectory discretization statistics (mean$\pm$std across trajectories).}
\label{tab:state_granularity}
\setlength{\tabcolsep}{5pt}
\begin{tabular}{llcccc}
\toprule
\multirow{2}{*}{\textbf{Encoder}} & \multirow{2}{*}{\textbf{Metric}}
    & \multicolumn{4}{c}{\textbf{Number of states} $|\mathcal{S}|$} \\
\cmidrule(lr){3-6}
    &  & 1024 & 2048 & 3072 & 4096 \\
\midrule
\multirow{3}{*}{DINOv2-S} & Coverage (\%) & $0.64{\scriptstyle\pm0.40}$ & $0.37{\scriptstyle\pm0.23}$ & $0.26{\scriptstyle\pm0.17}$ & $0.21{\scriptstyle\pm0.13}$ \\
    & Avg.\ revisits & $8.0{\scriptstyle\pm6.1}$ & $7.1{\scriptstyle\pm5.7}$ & $6.8{\scriptstyle\pm5.6}$ & $6.5{\scriptstyle\pm5.6}$ \\
    & Singleton (\%) & $21.7{\scriptstyle\pm19.1}$ & $25.5{\scriptstyle\pm19.9}$ & $27.6{\scriptstyle\pm20.5}$ & $28.9{\scriptstyle\pm20.9}$ \\
\midrule
\multirow{3}{*}{DINOv2-B} & Coverage (\%) & $0.69{\scriptstyle\pm0.42}$ & $0.40{\scriptstyle\pm0.25}$ & $0.28{\scriptstyle\pm0.18}$ & $0.22{\scriptstyle\pm0.14}$ \\
    & Avg.\ revisits & $7.6{\scriptstyle\pm5.9}$ & $6.6{\scriptstyle\pm5.4}$ & $6.3{\scriptstyle\pm5.3}$ & $6.0{\scriptstyle\pm5.1}$ \\
    & Singleton (\%) & $23.4{\scriptstyle\pm19.3}$ & $27.4{\scriptstyle\pm20.2}$ & $29.6{\scriptstyle\pm20.7}$ & $30.9{\scriptstyle\pm21.0}$ \\
\midrule
\multirow{3}{*}{DINOv2-L} & Coverage (\%) & $0.69{\scriptstyle\pm0.42}$ & $0.40{\scriptstyle\pm0.25}$ & $0.29{\scriptstyle\pm0.18}$ & $0.23{\scriptstyle\pm0.14}$ \\
    & Avg.\ revisits & $7.4{\scriptstyle\pm5.9}$ & $6.5{\scriptstyle\pm5.4}$ & $6.2{\scriptstyle\pm5.3}$ & $5.9{\scriptstyle\pm5.2}$ \\
    & Singleton (\%) & $23.4{\scriptstyle\pm19.1}$ & $27.1{\scriptstyle\pm20.0}$ & $29.2{\scriptstyle\pm20.4}$ & $30.8{\scriptstyle\pm20.7}$ \\
\midrule
\multirow{3}{*}{SigLIP2-B} & Coverage (\%) & $0.38{\scriptstyle\pm0.25}$ & $0.22{\scriptstyle\pm0.15}$ & $0.16{\scriptstyle\pm0.11}$ & $0.13{\scriptstyle\pm0.09}$ \\
    & Avg.\ revisits & $13.3{\scriptstyle\pm9.6}$ & $11.8{\scriptstyle\pm8.9}$ & $11.0{\scriptstyle\pm8.4}$ & $10.5{\scriptstyle\pm8.1}$ \\
    & Singleton (\%) & $14.7{\scriptstyle\pm18.5}$ & $17.5{\scriptstyle\pm19.2}$ & $19.0{\scriptstyle\pm19.6}$ & $20.0{\scriptstyle\pm19.9}$ \\
\bottomrule
\end{tabular}
\end{table}

\paragraph{Results.}
Test log-likelihood improves monotonically with $K$ (Figure~\ref{fig:bridge_results}): from $-1.81$ at $K{=}2$ to $-0.82$ at $K{=}6$ with no saturation, unlike the labyrinth where gains plateau by $K{=}4$. IntentionRNN and IntentionLSTM perform similarly; IntentionTransformer lags at small $K$ but narrows the gap as $K$ increases, reproducing the same ranking observed on the labyrinth and supporting IntentionRNN as a robust default. We report visualizations at $K{=}4$ because Bridge~V2 trajectories are relatively short; at $K{=}5$ and beyond, segments become excessively brief and lose interpretability, with new classes capturing transient gripper adjustments rather than distinct behavioral modes.

Figure~\ref{fig:arm_intentions} visualizes per-timestep intention assignments ($K{=}4$, DINOv2-S). Without any supervision, PRISM recovers four classes: APPROACH/DEPART (arm moving toward or away from target), GRASP (gripper closing), CARRY (transporting object), and IDLE (minor adjustments). The boundaries are sharp and consistent across trajectories with different objects and environments, confirming temporally coherent segmentation. Additional examples are provided in \S\ref{app:bridge_extra}.

\begin{figure}
    \centering
    \includegraphics[width=0.99\linewidth]{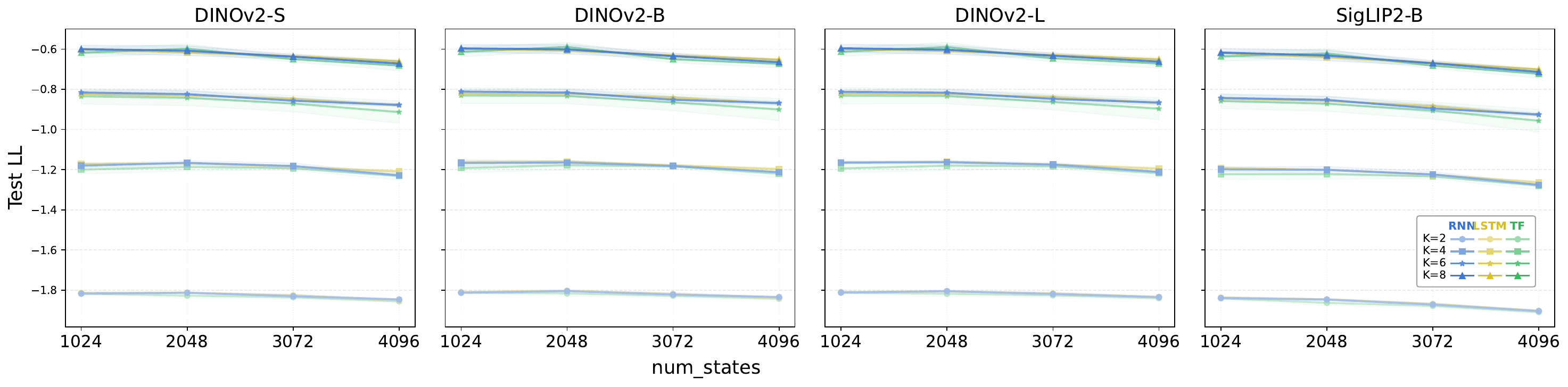}
    \caption{BridgeData~V2: test log-likelihood by encoder, intention network, and number of latents $K$.}
    \label{fig:bridge_results}
\end{figure}

\begin{figure}
    \centering
    \includegraphics[width=\linewidth]{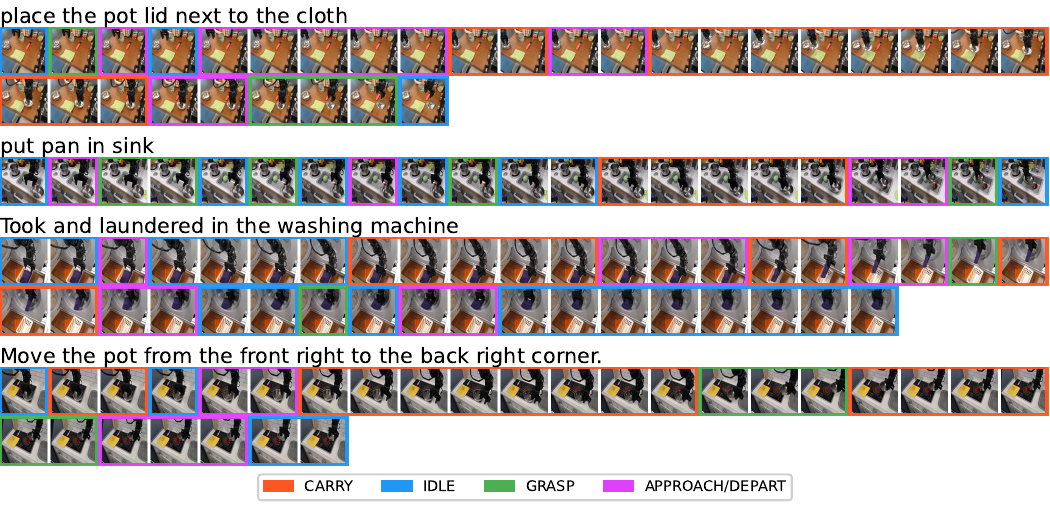}
    \caption{Per-timestep intention assignments on BridgeData~V2 trajectories. Frame borders are color-coded by predicted intention: orange (CARRY), blue (IDLE), green (GRASP), magenta (APPROACH/DEPART).}
    \label{fig:arm_intentions}
\end{figure}

\section{Conclusion and Discussion}
We presented PRISM, an EM-based framework for multi-intention IRL that parameterizes intention dynamics through a lightweight recurrent neural network. By replacing the Markov chain of HIQL and the fixed-window augmentation of SWIRL with an end-to-end learned recurrent mapping, PRISM adapts to the temporal structure of each dataset without manual design choices, while retaining closed-form reward recovery. The proven EM decomposition separates the reward and gating subproblems exactly, and the $\mathcal{O}(nK)$ E-step scales gracefully. Experiments on three progressively complex domains, from an abstract gridworld validating non-Markovian theory, through real mouse behavior confirming biological interpretability, to large-scale robotic manipulation demonstrating cross-domain generality, show that PRISM consistently achieves the highest held-out likelihood while recovering nameable intentions from unlabeled data. The recovered reward maps provide a structured characterization of the latent goals behind observed behavior, revealing not just \emph{what} an agent did but \emph{which objective} it was pursuing at each moment.

\paragraph{Limitations and future work.}
(a)~\emph{Tabular assumption.} PRISM currently requires a discrete MDP. Our $k$-means pipeline bridges this gap for visual domains but introduces quantization error. Learned codebooks such as VQ-BeT~\citep{lee2024behavior} or SAQ~\citep{luo2023saq} could adapt cluster boundaries to maximize reward-recovery quality. However, discretization has fundamental limits: even cutting-edge vision-language-action models such as $\pi_{0.5}$~\citep{physical2025pi05} employ action tokenizers like FAST for pre-training their language-model backbone, yet such tokenization serves the model architecture rather than enabling fine-grained dexterous control directly. The continuous structure that dexterous manipulation demands is inherently lost through any discretization, which motivates moving beyond tabular methods entirely.
(b)~\emph{Model-free extension.} The natural path forward is to replace IAVI with model-free inverse Q-learning~\citep{NEURIPS2020_a4c42bfd}, which operates directly in continuous state--action spaces via function approximation and does not require an explicit transition matrix. This would remove both the discretization bottleneck of~(a) and the model-based dependency, enabling PRISM to scale to dexterous robotic domains where tabular representations are infeasible.
(c)~\emph{Online extension.} PRISM operates offline over fully observed trajectories, which suffices for post-hoc behavioral analysis but precludes real-time intention monitoring. Introducing a probabilistic transition prior $P(z_t \mid z_{t-1}, \varphi_{t-1})$ would enable sequential filtering for online applications. In preliminary experiments, such a Bayesian variant achieved comparable log-likelihood but produced less interpretable reward maps, suggesting that the prior structure requires careful design to preserve reward quality.
(d)~\emph{Model selection.} Test log-likelihood improves monotonically with $K$, but beyond the true number of behavioral modes the additional intentions fragment trajectories into unnameable segments, particularly when trajectories are short (\S\ref{sec:bridge-v2}). A single-layer RNN with 128 hidden units already saturates performance on both datasets (Table~\ref{tab:model_type}), suggesting that the switching dynamics in these domains are low-dimensional. Longer, more complex demonstrations would be needed to justify both larger $K$ and more expressive sequence models.

\section*{Acknowledgments}
This work has been funded as part of BrainLinks-BrainTools, which is funded by the Federal Ministry of Economics, Science and Arts of Baden-W\"urttemberg within the sustainability program for projects of the Excellence Initiative II, and CRC/TRR 384 ``IN-CODE''.

\bibliographystyle{plainnat}
\bibliography{refs}

\newpage
\appendix
\setcounter{equation}{0}
\numberwithin{equation}{section}
\setcounter{figure}{0}
\numberwithin{figure}{section}
\setcounter{table}{0}
\numberwithin{table}{section}

\section{Theoretical and Technical Details}

\subsection{IAVI Formulation}
\label{app:iavi}
Given expert demonstrations $\mathcal{D}$, the IRL problem under a Boltzmann policy is formulated as:
\begin{equation}\label{eq:irl_mle}
\begin{array}{ll}
\text{maximize} & \expect_{(\xi,\psi) \sim (\mathcal{D},\mathcal{O})}  \log \prob \left( \xi \mid \pi_r \right) \\
\text{subject to} & \pi_r(a \mid s) = \exp\left( Q(s, a) - \log \textstyle\sum_{a' \in \mathcal{A}} \exp Q(s, a') \right)\\
& Q(s, a) = r(s, a) + \gamma \textstyle\sum_{s'} P(s' \mid s, a) \max_{a' \in \mathcal{A}} Q(s', a')\\
&s \in \mathcal{S},\quad a \in \mathcal{A}
\end{array}
\end{equation}
where $r$ is the optimization variable. When $P$ is known, this can be solved in closed form via least squares, yielding IAVI~\citep{NEURIPS2020_a4c42bfd}.

\subsection{Proof of Theorem~\ref{thm:decomp}}
\label{app:proof}
\begin{proof}
The objective function $J(\Theta^+ \mid \Theta)$ from problem~\eqref{eq:em_mle} can be written as:
\begin{align}
    &J(\Theta^+ \mid \Theta)\\
    =\quad& \expect_{(\xi,\psi) \sim (\mathcal{D},\mathcal{O}),\, \eta}\log \prob(\xi, \eta \mid \psi, \Theta^+) \notag\\
    =\quad& \expect_{(\xi,\psi) \sim (\mathcal{D},\mathcal{O})}
       \left(\sum_\eta \prob(\eta \mid \xi, \psi, \Theta)\;
             \log \prob(\xi, \eta \mid \psi, \Theta^+)\right) \notag\\
    =\quad& \expect_{(\xi,\psi) \sim (\mathcal{D},\mathcal{O})}
       \left(\sum_\eta \prob(\eta \mid \xi, \psi, \Theta)\;
             \sum_{i=1}^n \log \prob(z_i^+ {=} z_i \mid \varphi_i, \theta^+)\,
             \prob\bigl((s_i, a_i) \mid r_{z_i}^+\bigr)\right) \notag\\
    =\quad& \expect_{(\xi,\psi) \sim (\mathcal{D},\mathcal{O})}
       \left(\sum_\eta \prob(\eta \mid \xi, \psi, \Theta)\;
             \sum_{i=1}^n \sum_{k=1}^K \mathbb{I}_k(z_i)\,
             \log \prob(z_i^+ {=} k \mid \varphi_i, \theta^+)\,
             \prob\bigl((s_i, a_i) \mid r_k^+\bigr)\right) \notag\\
    =\quad& \expect_{(\xi,\psi) \sim (\mathcal{D},\mathcal{O})}
       \Bigg(\sum_{k=1}^K \sum_{i=1}^n
             \underbrace{\sum_\eta \prob(\eta \mid \xi, \psi, \Theta)\,\mathbb{I}_k(z_i)}_{\displaystyle = \,\prob(z_i = k \mid \xi, \psi, \Theta)} \log \prob(z_i^+ {=} k \mid \varphi_i, \theta^+)\,
             \prob\bigl((s_i, a_i) \mid r_k^+\bigr)\Bigg) \notag\\
    =\quad& \underbrace{\expect_{(\xi,\psi) \sim (\mathcal{D},\mathcal{O})}
       \left(\sum_{k=1}^K \sum_{i=1}^n
             \prob(z_i {=} k \mid \xi, \psi, \Theta)\;
             \log \prob(z_i^+ {=} k \mid \varphi_i, \theta^+)\right)}_{\text{(I): depends only on }\theta^+} \label{eq:proof_7a}\\
    &\quad + \underbrace{\expect_{(\xi,\psi) \sim (\mathcal{D},\mathcal{O})}
       \left(\sum_{k=1}^K \sum_{i=1}^n
             \prob(z_i {=} k \mid \xi, \psi, \Theta)\;
             \log \prob\bigl((s_i, a_i) \mid r_k^+\bigr)\right)}_{\text{(II): depends only on }\mathcal{R}^+}, \label{eq:proof_7b}
\end{align}
where $\mathbb{I}_k$ is the indicator function with $\mathbb{I}_k(x) = 1$ for $x = k$ and $0$ otherwise. Thus maximizing $J(\Theta^+ \mid \Theta)$ over $\Theta^+$ is equivalent to separately maximizing~\eqref{eq:proof_7a} over~$\theta^+$ and~\eqref{eq:proof_7b} over~$\mathcal{R}^+ = \{r_1^+, \dots, r_K^+\}$.

By Assumption~\ref{asm:control}, ${f_\theta(\varphi)}_k = \prob(r_k \mid \varphi)$, so the first optimization problem becomes~\eqref{eq:theta_opt}. In~\eqref{eq:proof_7b}, distinct~$k$ share no parameters, so the maximization decomposes into $K$ independent subproblems. Each has the same structure as the single-intention IRL problem~\eqref{eq:irl_mle}, with the $i$-th demonstration weighted by $\prob(z_i = k \mid \xi, \psi, \Theta)$, yielding~\eqref{eq:reward_opt}. The Boltzmann-policy constraints are introduced to make each subproblem tractable via IAVI.
\end{proof}

\subsection{Intention Network Architecture Details}
\label{app:architecture}
Each step $(s_t, a_t)$ is encoded by two independent learned embedding matrices $E_s \in \reals^{|\mathcal{S}| \times d}$ and $E_a \in \reals^{|\mathcal{A}| \times d}$, producing input $x_t = E_s(s_t) + E_a(a_t)$. Batches of variable-length trajectories are zero-padded with a binary mask. The recurrent variants (RNN, LSTM) use a single recurrent layer with hidden dimension $d'$; the Transformer variant applies sinusoidal positional encoding and a standard encoder with a padding mask. In all variants, $K$ logits are converted to $f_\theta(\varphi_t) \in \Delta^{K}$ via softmax.

\section{Gridworld: Full Results}
\label{app:gridworld}

\begin{table}
    \centering
    \begin{tabular}{lcccc}
        \toprule
        \multirow{2}{*}{\textbf{Method}}
            & \multicolumn{2}{c}{\textbf{EVD (MAE)}}
            & \multicolumn{2}{c}{\textbf{EVD ($s_0$)}} \\
        \cmidrule(lr){2-3} \cmidrule(lr){4-5}
            & goal & abandon & goal & abandon \\
        \midrule
        PRISM ($K{=}2$)       & $\mathbf{2.40 \pm 0.37}$ & $\mathbf{5.60 \pm 0.45}$ & $\mathbf{-1.74 \pm 0.59}$ & $\mathbf{-6.02 \pm 0.85}$ \\
        HIQL ($K{=}2$)        & $2.51 \pm 0.13$ & $6.12 \pm 0.38$ & $-1.95 \pm 0.54$ & $-6.80 \pm 1.25$ \\
        IAVI ($K{=}1$)         & $2.06 \pm 0.02$ & $6.74 \pm 0.00$ & $-1.41 \pm 0.01$ & $-9.20 \pm 0.00$ \\
        MaxCausalEnt ($K{=}1$) & $5.32 \pm 0.00$ & $6.67 \pm 0.00$ & $-3.32 \pm 0.00$ & $-9.12 \pm 0.00$ \\
        MaxEnt ($K{=}1$)       & $6.61 \pm 0.00$ & $6.76 \pm 0.00$ & $-4.16 \pm 0.00$ & $-9.10 \pm 0.00$ \\
        \bottomrule
    \end{tabular}
    \caption{Expected value difference on the frustration gridworld (5-fold CV). Closer to zero is better.}
    \label{tab:evd}
\end{table}

\begin{figure}
    \centering
    \includegraphics[width=0.98\linewidth]{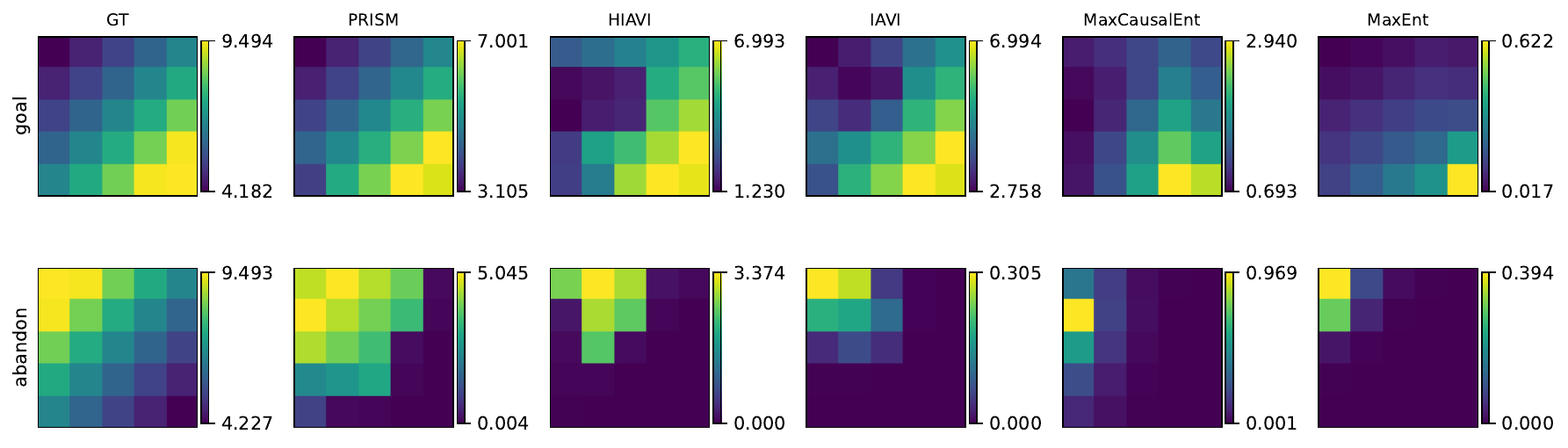}
    \caption{Full state-value heatmaps on the frustration gridworld for all methods.}
    \label{fig:grid_heatmaps_full}
\end{figure}

\begin{figure}
    \centering
    \includegraphics[width=0.94\linewidth]{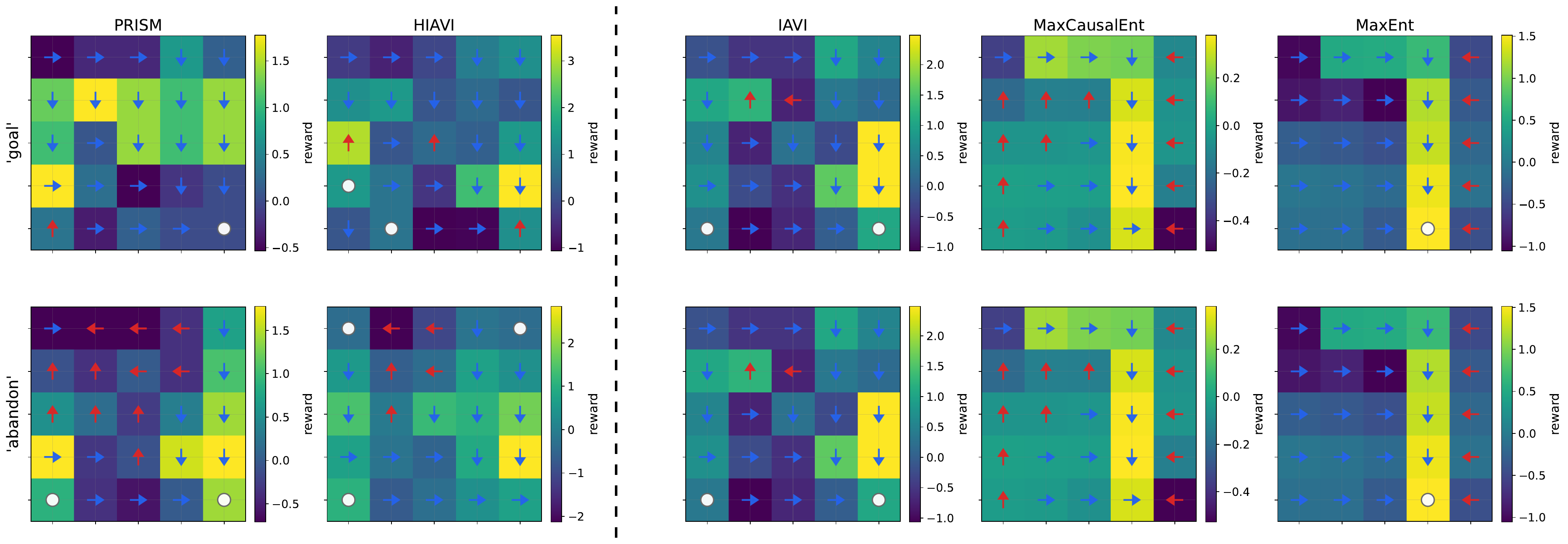}
    \caption{Recovered per-state reward maps and greedy actions on the frustration gridworld.}
    \label{fig:grid_rewardmaps_full}
\end{figure}

\clearpage
\section{Default Hyperparameters}
\label{app:hparams}

\begin{table}
\centering
\small
\caption{Default hyperparameters for the labyrinth and Bridge~V2 experiments.}
\label{tab:hparams}
\setlength{\tabcolsep}{4pt}
\begin{tabular}{llccc}
\toprule
\textbf{Category} & \textbf{Hyperparameter} & \textbf{Symbol} & \textbf{Labyrinth} & \textbf{Bridge V2} \\
\midrule
\multirow{3}{*}{MDP}
    & Num.\ states     & $|\mathcal{S}|$ & 127     & 2048     \\
    & Num.\ actions    & $|\mathcal{A}|$ & 4       & 32      \\
    & Discount factor      & $\gamma$        & 0.97    & 0.97    \\
\midrule
\multirow{3}{*}{EM}
    & Max EM iterations    &        & 180     & 150     \\
    & Random seed          &        & 42      & 42      \\
\midrule
\multirow{2}{*}{Intention model}
    & Num.\ latent intentions   & $K$    & 3       & 4       \\
    & Architecture         &        & IntentionRNN     & IntentionRNN \\
\midrule
\multirow{5}{*}{Intention net}
    & Embedding dim.\ & $d$    & 128     & 128     \\
    & RNN\,/\,LSTM hidden dim.\ & $d'$   & 128     & 128     \\
    & Num.\ layers     &        & 1       & 1       \\
    & Attn.\ heads (Transformer) &        & 4       & 4       \\
\midrule
\multirow{3}{*}{Optimiser}
    & Learning rate        &        & $10^{-3}$        & $10^{-3}$    \\
    & RNN\,/\,LSTM epochs per M-step    &        & 1       & 1       \\
    & Transformer epochs per M-step    &        & 8       & 8       \\
\midrule
\multirow{3}{*}{Regularisation}
    & Regularisation type  &        & KL + L1 &         \\
    & L1 smoothness weight & $\lambda_{\ell_1}$ & 2.22     & 0.0     \\
    & KL smoothness weight & $\lambda_\text{kl}$ & 1.48     & 0.0     \\
\bottomrule
\end{tabular}
\end{table}

\section{Ablation Experiments}
\label{app:ablation}

\begin{table}
\centering
\begin{tabular}{lcccc}
\toprule
\multirow{2}{*}{$K$}
    & \multicolumn{2}{c}{\textbf{Labyrinth}}
    & \multicolumn{2}{c}{\textbf{Bridge V2}} \\
\cmidrule(lr){2-3} \cmidrule(lr){4-5}
    & Train LL & Test LL & Train LL & Test LL \\
\midrule
$1$ & -0.86801 & -0.87071 & -2.45171 & -2.52187 \\
$2$ & -0.75432 & -0.76024 & -1.74644 & -1.81282 \\
$3$ & -0.63754 & -0.64624 & -1.36307 & -1.43000 \\
$4$ & -0.58141 & -0.58714 & -1.09903 & -1.16584 \\
$5$ & -0.53726 & -0.54573 & -0.90255 & -0.96912 \\
$6$ & -0.51220 & -0.52129 & -0.75780 & -0.82373 \\
\bottomrule
\end{tabular}
\caption{Log-likelihood vs number of latent intentions $K$ (IntentionRNN, 5-fold CV, 3 random seeds).}
\label{tab:num_latents}
\end{table}

\begin{table}
    \centering
    \begin{tabular}{lcccc}
    \toprule
    \multirow{2}{*}{\textbf{Model}}
        & \multicolumn{2}{c}{\textbf{Labyrinth}}
        & \multicolumn{2}{c}{\textbf{Bridge V2}} \\
    \cmidrule(lr){2-3} \cmidrule(lr){4-5}
        & Train LL & Test LL & Train LL & Test LL \\
    \midrule
    IntentionRNN         & -0.63754 & -0.64624 & -1.09903 & -1.16584 \\
    IntentionLSTM        & -0.60993 & -0.62705 & -1.10082 & -1.16839 \\
    IntentionTransformer & -0.86797 & -0.87067 & -1.12025 & -1.18638 \\
    \bottomrule
    \end{tabular}
    \caption{Log-likelihood by intention network architecture ($d{=}128$, single layer, 5-fold CV).}
    \label{tab:model_type}
\end{table}

\clearpage
\section{BridgeData~V2: Additional Intention Assignments}
\label{app:bridge_extra}
\begin{figure}
    \centering
    \includegraphics[width=\linewidth]{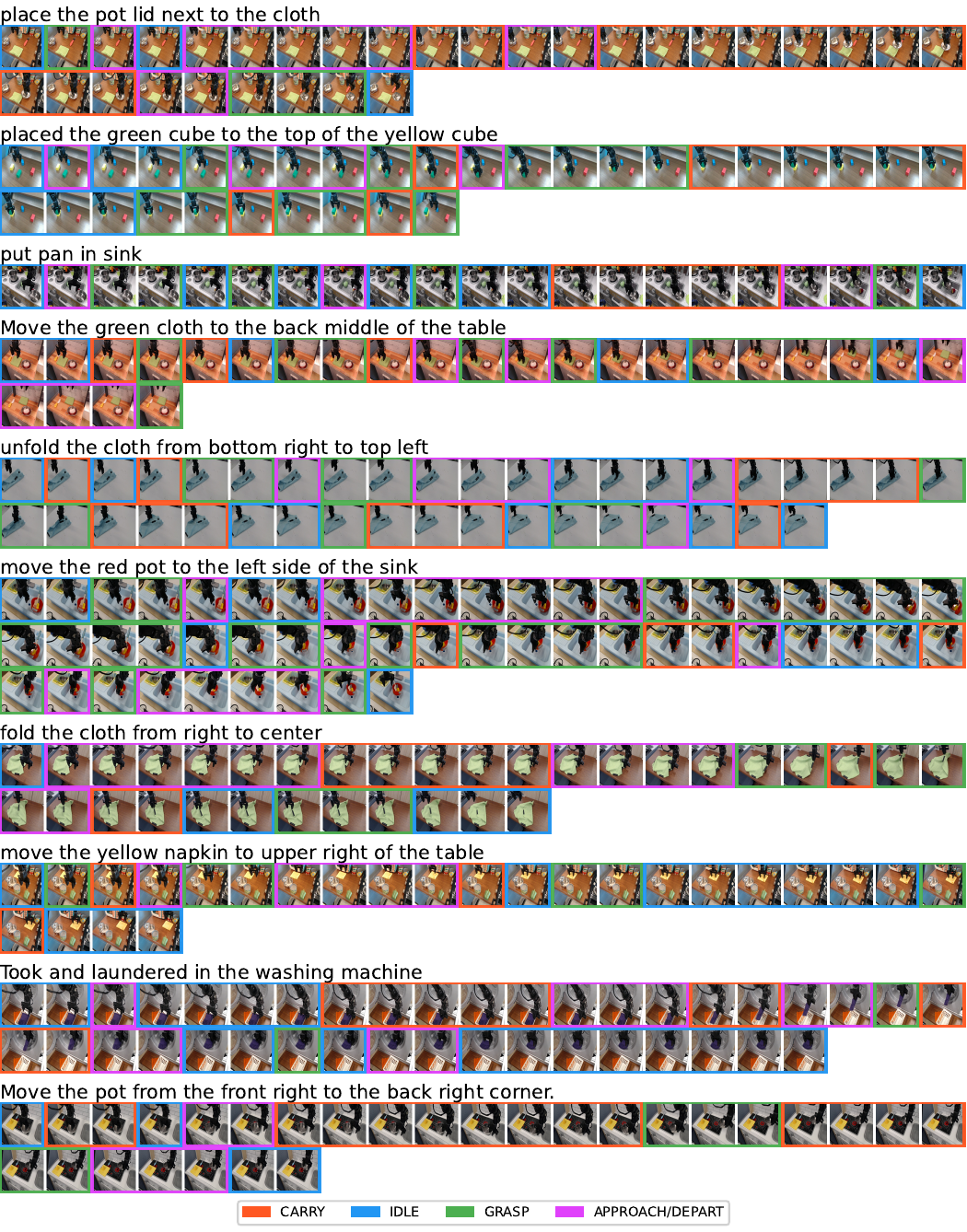}
    \caption{More trajectories of intention assignments.}
    \label{fig:arm_classes_lush_dino}
\end{figure}


\end{document}